\def\year{2019}\relax
\documentclass[letterpaper]{article} 
\usepackage{aaai19}  
\usepackage{times}  
\usepackage{helvet}  
\usepackage{courier}  
\usepackage{url}  
\usepackage{graphicx}  
\usepackage{booktabs}       
\usepackage{amsfonts}       
\usepackage{microtype}      
\usepackage{subcaption}
\usepackage{multirow}
\usepackage{algorithm}
\usepackage{algorithmicx}
\usepackage{algpseudocode}
\usepackage{amsmath,amssymb,amsfonts,amsthm}
\usepackage{mathtools}
\usepackage{adjustbox}
\usepackage{color}
\usepackage{makecell}

\newtheorem{theorem}{Theorem}

\newcommand{\bx}{\mathbf{x}}

\newcommand{\bw}{\mathbf{w}}
\newcommand{\bg}{\mathbf{g}}
\newcommand{\bu}{\mathbf{u}}
\newcommand{\bv}{\mathbf{v}}
\newcommand{\be}{\mathbf{e}}
\newcommand{\bdelta}{\boldsymbol{\delta}}

\newcommand{\Dist}{\textnormal{Dist}}
\newcommand{\Loss}{\textnormal{Loss}}

\begin{document}
%
\title{AutoZOOM: Autoencoder-based Zeroth Order Optimization Method \\ for Attacking  Black-box Neural Networks}
\author{Chun-Chen Tu$^{1}$\thanks{equal contribution},~Paishun Ting$^{1*}$,~Pin-Yu Chen$^{2*}$,~Sijia Liu$^{2}$, \\ {\Large \textbf{Huan Zhang}$^{3}$,~ \textbf{Jinfeng Yi}$^{4}$,~\textbf{Cho-Jui Hsieh}$^{3}$,~\textbf{Shin-Ming Cheng}$^{5}$ }\\
	$^1$University of Michigan, Ann Arbor, USA \\
	$^2$MIT-IBM Watson AI Lab, IBM Research \\
   $^3$University of California, Los Angeles, USA \\ 
	$^4$JD AI Research, Beijing, China \\  
	$^5$National Taiwan University of Science and Technology, Taiwan 
}
\nocopyright
\maketitle
\begin{abstract}
Recent studies have shown that adversarial examples in state-of-the-art image classifiers trained by deep neural networks (DNN) can be easily generated when the target model is transparent to an attacker, known as the white-box setting. However, when attacking a deployed machine learning service, one can only acquire the input-output correspondences of the target model; this is the so-called black-box attack setting. The major drawback of existing black-box attacks is the need for excessive model queries, which may give a false sense of model robustness due to inefficient query designs. To bridge this gap, we propose a generic framework for query-efficient black-box attacks. Our framework, \textbf{AutoZOOM}, which is short for \textbf{Auto}encoder-based \textbf{Z}eroth \textbf{O}rder \textbf{O}ptimization \textbf{M}ethod, has two novel building blocks towards efficient black-box attacks: (i) an adaptive random gradient estimation strategy to balance query counts and distortion, and (ii) an autoencoder that is either trained offline with unlabeled data or a bilinear resizing operation for attack acceleration. 
Experimental results suggest that, by applying AutoZOOM to a state-of-the-art black-box attack (ZOO), a significant reduction in model queries can be achieved without sacrificing the attack success rate and the visual quality of the resulting adversarial examples. In particular, when compared to the standard ZOO method,
AutoZOOM can consistently reduce the mean query counts in finding successful adversarial examples (or reaching the same distortion level)
by at least 93\% on MNIST, CIFAR-10 and ImageNet datasets, leading to novel insights on adversarial robustness. 
\end{abstract}

\section{Introduction}

In recent years, ``machine learning as a service'' has offered the world an effortless access to powerful machine learning tools for a wide variety of tasks.
For example, commercially available services such as Google Cloud Vision API
and Clarifai.com
provide well-trained image classifiers to the public. 
One is able to upload and obtain the class prediction results for images at hand at a low price. 
However, the existing and emerging machine learning platforms and their low model-access costs raise ever-increasing security concerns, as they also offer an ideal environment for testing malicious attempts. Even worse, the risks can be amplified when 
these services are used to build derived products such that the 
inherent security vulnerability could be leveraged by attackers.


In many computer vision tasks,  
DNN models achieve the state-of-the-art prediction accuracy and hence are widely deployed in modern machine learning services. 
Nonetheless, recent studies have highlighted DNNs' vulnerability to adversarial perturbations. In the \textit{white-box} setting in which the target model is entirely transparent to an attacker,  
visually imperceptible adversarial images can be easily crafted to fool a target DNN model towards misclassification by leveraging the input gradient information \cite{szegedy2013intriguing,goodfellow2014explaining}. However, in the \textit{black-box} setting in which
the parameters of the deployed model are hidden and one can only observe the input-output correspondences of a queried example, crafting adversarial examples requires a gradient-free (zeroth order) optimization approach to gather necessary attack information. Figure \ref{fig:first_figure} displays a prediction-evasive adversarial example crafted via iterative model queries from a black-box DNN (the Inception-v3 model \cite{szegedy2016rethinking}) trained on ImageNet.



\begin{figure}[!t]
	\centering
	\includegraphics[width=1\linewidth]{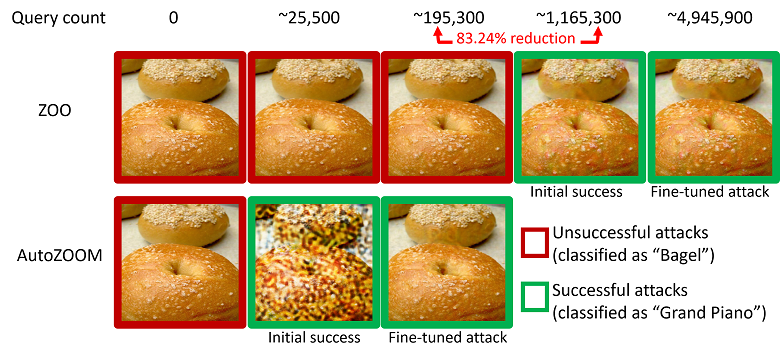}
		\caption{AutoZOOM significantly reduces the number of queries required to generate a successful adversarial Bagel image from the black-box Inception-v3 model.}
		\label{fig:first_figure}
\end{figure}

\begin{figure*}[t]
	\centering
	\includegraphics[width=1\textwidth]{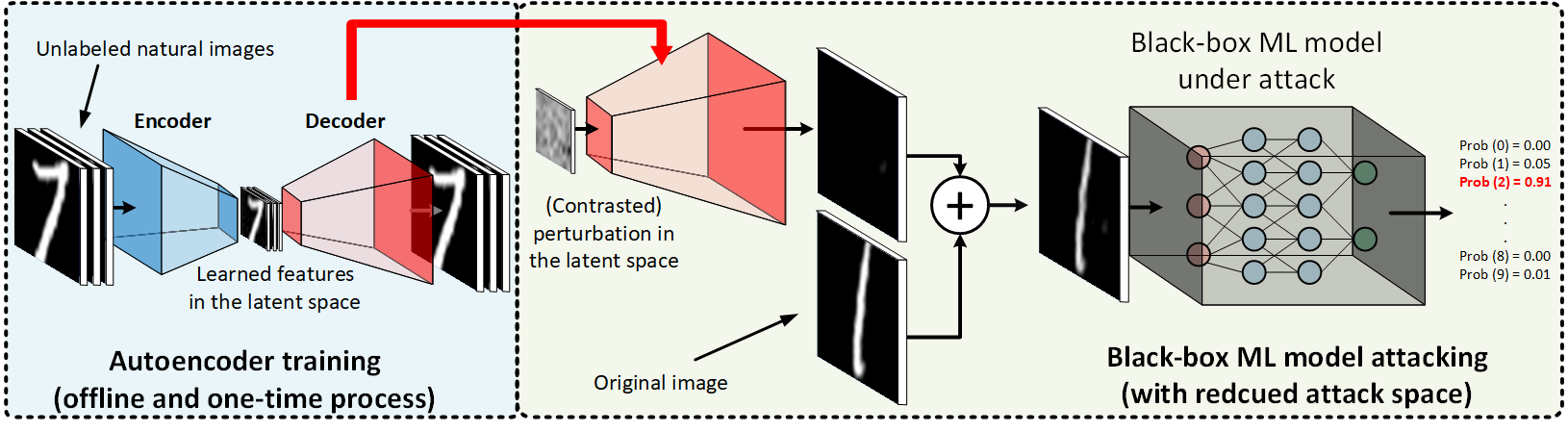}
	\caption{Illustration of attack dimension reduction through a ``decoder'' in AutoZOOM for improving query efficiency in black-box attacks. 
	The decoder has two modes: (i)
	An autoencoder (AE) trained on unlabeled natural images that are different from the attacked images and training data; (ii) a simple bilinear image resizer (BiLIN) that is applied channel-wise to extrapolate low-dimensional feature to the original image dimension (width $\times$ height). In the latter mode, no additional training is required.}
	\label{fig:autoencoder}
\end{figure*}

Albeit achieving remarkable attack effectiveness by the use of gradient estimation, current black-box attack methods, such as \cite{chen2017zoo,bhagoji2017exploring}, are not query-efficient since they exploit coordinate-wise gradient estimation and value update, which inevitably incurs an excessive number of model queries and may give a false sense of model robustness due to inefficient query designs.
In this paper, we propose to tackle the preceding problem by using \textbf{AutoZOOM}, an \textbf{Auto}encoder-based \textbf{Z}eroth \textbf{O}rder \textbf{O}ptimization \textbf{M}ethod.
AutoZOOM has two novel building blocks: (i) a new and adaptive random gradient estimation strategy to balance the query counts and distortion when crafting adversarial examples, and (ii) an autoencoder that is either trained offline on other unlabeled data, or based on a simple bilinear resizing operation, in order to accelerate black-box attacks. As illustrated in Figure \ref{fig:autoencoder}, AutoZOOM utilizes a ``decoder'' to craft a high-dimensional adversarial perturbation from the (learned) low-dimensional latent-space representation, and its query efficiency can be well explained by the dimension-dependent convergence rate in gradient-free optimization.

\noindent{\textbf{Contributions.}} We summarize our main contributions and new insights on adversarial robustness as follows:
\begin{enumerate}
	\item We propose AutoZOOM, a novel query-efficient black-box attack framework for generating adversarial examples. AutoZOOM features an adaptive random gradient estimation strategy and dimension reduction techniques (either an offline trained autoencoder or a bilinear resizer)
	to reduce attack query counts while maintaining attack effectiveness and visual similarity. To the best of our knowledge, AutoZOOM is the first black-box attack using random full gradient estimation and data-driven acceleration. 
	\item We use the convergence rate of zeroth-order optimization to motivate the query efficiency of AutoZOOM and provide an error analysis of the new gradient estimator in AutoZOOM to the true gradient for characterizing the trade-offs between estimation error and query counts.
	\item When applied to a state-of-the-art black-box attack proposed in \cite{chen2017zoo}, AutoZOOM attains a similar attack success rate while achieving a significant reduction (at least 93\%) in the mean query counts required to attack the DNN image classifiers for MNIST, CIFAR-10 and ImageNet. It can also fine-tune the distortion in the post-success stage by performing finer gradient estimation. 
	\item In the experiments, we also find that AutoZOOM with a simple bilinear resizer as the decoder (AutoZOOM-BiLIN) can attain noticeable query efficiency, despite that it is still worse than AutoZOOM with an offline trained autoencoder (AutoZOOM-AE). However, AutoZOOM-BiLIN is easier to be mounted as no additional training is required. The results also suggest an interesting finding that while learning effective low-dimensional representations of legitimate images is still a challenging task, black-box attacks using significantly less degree of freedoms (i.e., reduced dimensions) are certainly plausible.
\end{enumerate}

\section{Related Work}
Gradient-based adversarial attacks on DNNs fall within the white-box setting, since acquiring the gradient with respect to the input requires knowing the weights of the target DNN. 
As a first attempt towards black-box attacks, the authors in \cite{papernot2017practical} proposed to train a substitute model using iterative model queries, performing white-box attacks on the substitute model, and implementing transfer attacks to the target model \cite{papernot2016transferability,liu2016delving}. However, its 
attack performance can be severely degraded due to poor attack transferability~\cite{DBLP:conf/eccv/SuZCYCG18}. 
Although ZOO achieves a similar attack success rate and comparable visual quality as many white-box attack methods \cite{chen2017zoo}, its coordinate-wise gradient estimation  requires excessive target model evaluations and is hence not query-efficient. The same gradient estimation technique is also used in  \cite{bhagoji2017exploring}. 

Beyond optimization-based approaches, the authors in \cite{ilyas2018black} proposed to use a natural evolution strategy (NES) to enhance  query efficiency. Although there is a vector-wise gradient estimation step in the NES attack, we treat it as a parallel work since its natural evolutionary step is out of the scope of black-box attacks using zeroth-order gradient descent.  We also note that different from NES, our AutoZOOM framework uses a theory-driven query-efficient random-vector based gradient estimation strategy. In addition, AutoZOOM
could be applied to further improve the query efficiency of NES, since NES does not take into account the factor of attack dimension reduction, which is the novelty in AutoZOOM as well as the main focus of this paper.

Under a more restricted attack setting, where only the decision (top-1 prediction class) is known to an attacker,
the authors in \cite{brendel2017decision} proposed a random-walk based attack around the decision boundary.
 Such a black-box attack dispenses class prediction scores and hence requires additional model queries. 
 Due to space limitation, we provide more background and a table comparing existing black-box attacks in the supplementary material.



\section{AutoZOOM: Background and Methods}
\label{sec_auto}
\subsection{Black-box Attack Formulation and Zeroth Order Optimization}
 Throughout this paper, we focus on improving the query efficiency of gradient-estimation and gradient-descent based black-box attacks empowered by AutoZOOM, and we consider the threat model that the class prediction scores are known to an attacker. 
In this setting, it suffices to denote the target DNN as a classification function $F: [0,1]^{d} \mapsto \mathbb{R}^K$ that takes a $d$-dimensional scaled image as its input and yields a vector of prediction scores of all $K$ image classes, such as the prediction probabilities for each class. We further consider the case of applying an entry-wise monotonic transformation $M(F)$ to the output of $F$ for black-box attacks, since monotonic transformation preserves the ranking of the class predictions and can alleviate the problem of large score variation in $F$ (e.g., probability to log probability).   

Here we formulate black-box targeted attacks. The formulation can be easily adapted to untargeted attacks.  
Let $(\bx_0,t_0)$ denote a natural image $\bx_0$ and its ground-truth class label $t_0$, and let ($\bx,t$) denote the adversarial example of $\bx_0$ and the target attack class label $t \neq t_0$. 
The problem of finding an adversarial example can be formulated as an optimization problem taking the generic form of 
\begin{align}
\label{eqs_attack_optimization_problem}
& \text{min}_{\bx \in [0, 1]^d }~\Dist (\mathbf{x}, \mathbf{x}_0) + \lambda \cdot \Loss (\mathbf{x}, M(F(\bx)),t),
\end{align}
where $\Dist (\mathbf{x}, \mathbf{x}_0)$ measures the distortion between $\mathbf{x}$ and $\mathbf{x}_0$, $\Loss(\cdot)$ is an attack objective reflecting the likelihood of predicting $t = \arg \max_{k \in \{1,\ldots,K\} } [M(F(\bx))]_k$, $\lambda$ is a regularization coefficient, and the constraint $\mathbf{x} \in [0, 1]^d$ confines the adversarial image $\bx$ to the valid image space. The distortion $\Dist(\bx,\bx_0)$ is often evaluated by the $L_p$ norm defined as $\Dist(\bx,\bx_0)=\|\bx-\bx_0\|_p = \| \bdelta \|_p=\sum_{i=1}^{d} |\bdelta_i|^{1/p}$ for $p \geq 1$, where $\bdelta=\bx-\bx_0$ is the adversarial perturbation to $\bx_0$. 
The attack objective $\Loss(\cdot)$ can be the training loss of DNNs \cite{goodfellow2014explaining} or some designed loss based on model predictions \cite{carlini2017towards}.  

In the white-box setting, an adversarial example is generated by using downstream optimizers such as ADAM \cite{kingma2014adam} to solve (\ref{eqs_attack_optimization_problem}); this requires the gradient $\nabla f (\bx)$ of the objective function $f(\bx) = \Dist (\mathbf{x}, \mathbf{x}_0) + \lambda \cdot \Loss (\mathbf{x}, M(F(x)),t)$ relative to the input of $F$ via back-propagation in DNNs. However, in the black-box setting, acquiring $\nabla f (\cdot)$ is implausible, and one can only obtain the function evaluation $F(\cdot)$, which renders solving (\ref{eqs_attack_optimization_problem}) a zeroth order optimization problem.
Recently, zeroth order optimization approaches \cite{ghadimi2013stochastic,nesterov2017random,liu2017zeroth} circumvent the preceding challenge by approximating the true gradient via function evaluations. Specifically, in black-box attacks, the gradient estimate is applied to  both gradient computation and descent in the optimization process for solving (\ref{eqs_attack_optimization_problem}).

\subsection{Random Vector based Gradient Estimation}
\label{subsec:rand_grad_est}
As a first attempt to enable gradient-free black-box attacks on DNNs, the authors in \cite{chen2017zoo} use the symmetric difference quotient method \cite{lax2014calculus} to evaluate the gradient $\frac{\partial f(\bx)}{\partial \bx_i}$ of the $i$-th component by 
\begin{align}
\label{eqn_grad_coor}
g_i =  \frac{f(\bx+ h \be_i) - f(\bx - h \be_i) }{2h} \approx \frac{\partial f(\bx)}{\partial \bx_i}
\end{align}
using a small $h$. Here $\be_i$ denotes the $i$-th elementary basis. 
Albeit contributing to powerful black-box attacks and applicable to large networks like ImageNet, the nature of coordinate-wise gradient estimation step in (\ref{eqn_grad_coor}) must incur an enormous amount of model queries and is hence not query-efficient. For example, the ImageNet dataset has $ d=299\times 299 \times 3 \approx 270,000 $ input dimensions, rendering coordinate-wise zeroth order optimization based on gradient estimation query-inefficient.    

To improve query efficiency, we dispense with coordinate-wise estimation and instead propose a scaled random full gradient estimator of $\nabla f(\bx)$, defined as
\begin{align}
\label{eqn_grad_rand}
\bg =  b \cdot \frac{f(\bx+ \beta \bu) - f(\bx) }{\beta} \cdot \bu, 
\end{align}
where 
$\beta > 0$ is a smoothing parameter, $\bu$ is a unit-length vector that is uniformly drawn at random from a unit Euclidean sphere, and $b$ is a tunable scaling parameter that balances the bias and variance trade-off of the gradient estimation error. Note that with $b = 1$, the gradient estimator in \eqref{eqn_grad_rand} becomes the one used in \cite{duchi2015optimal}. With $b = d$, this estimator becomes the one adopted in \cite{gao2014information}. We will provide an optimal value $b^*$ for balancing query efficiency and estimation error in the following analysis.

\noindent{\textbf{Averaged random gradient estimation.}}
To effectively control the error in gradient estimation, we consider a more general gradient estimator, in which the gradient estimate is averaged over  $q$ random directions $\{ \mathbf u_j \}_{j = 1}^q$. That is,
\begin{align}
\label{eqn_grad_rand_avg}
\overline{\bg} =  \frac{1}{q} \sum_{j=1}^q \bg_j, 
\end{align}
where $\bg_j$ is a gradient estimate defined in \eqref{eqn_grad_rand} with $\mathbf u = \mathbf u_j$. 
The use of multiple random directions can reduce the variance of $\overline{\bg}$ 
in \eqref{eqn_grad_rand_avg} 
for convex loss functions \cite{duchi2015optimal,liu2017zeroth}.

Below we establish an error analysis of the averaged random gradient estimator in (\ref{eqn_grad_rand_avg}) for studying the influence of the parameters $b$ and $q$ on estimation error and query efficiency.

\begin{theorem}
	\label{thm_grad}
	Assume $f:\mathbb{R}^d \mapsto \mathbb{R} $ is differentiable and its gradient $\nabla f(\cdot)$ is $L$-Lipschitz\footnote{A function $W(\cdot)$ is $L$-Lipschitz if $\|W(\bw_1) - W(\bw_2)\|_2 \leq L \|\bw_1 - \bw_2\|_2$ for any $\bw_1, \bw_2$.  For DNNs with ReLU activations, $L$ can be derived from the model weights \cite{szegedy2013intriguing}.}. Then the mean squared estimation error of $\overline{\bg}$ in (\ref{eqn_grad_rand_avg}) is upper bounded by 
	\begin{align}
	\label{eqn_error}       
	\mathbb E \| \overline{\bg} - \nabla f(\mathbf x)\|_2^2  &\leq  4 ( \frac{b^2}{d^2}+\frac{ b^2 }{d q} + \frac{(b-d)^2}{d^2}) \| \nabla f(\mathbf x) \|_2^2  \nonumber \\
	&~~+ \frac{2q+1}{q} b^2  \beta^2 L^2.  
	\end{align}
\end{theorem}
\begin{proof}
	The proof is given in the supplementary file.
\end{proof}
Here we highlight the important implications based on Theorem \ref{thm_grad}: (i) The error analysis holds when $f$ is \textit{non-convex}; (ii) In DNNs, the true gradient $\nabla f$ can be viewed as the numerical gradient obtained via back-propagation; (iii) For any fixed $b$, selecting a small $\beta$ (e.g., we set $\beta=1/d$ in AutoZOOM) can effectively reduce the last error term in \eqref{eqn_error}, and we therefore focus on optimizing the first error term; (iv) The first error term in \eqref{eqn_error} exhibits the influence of $b$ and $q$ on the estimation error, and is independent of $\beta$.  We further elaborate on (iv) as follows.
Fixing $q$ and let $\eta(b)=\frac{b^2}{d^2}+\frac{ b^2 }{d q} + \frac{(b-d)^2}{d^2}$ to be the coefficient of the first error term in \eqref{eqn_error}, then the optimal $b$ that minimizes $\eta(b)$ is $b^* = \frac{dq}{2q+d}$. For query efficiency, one would like to keep $q$ small, which then implies $b^* \approx q$ and $\eta(b^*) \approx 1$ when the dimension $d$ is large. On the other hand, when $q \rightarrow \infty$, $b^* \approx d/2$  and $\eta(b^*) \approx 1/2$, which yields a smaller error upper bound but is query-inefficient. We also note that by setting $b = q$, the coefficient $\eta(b)=\frac{b^2}{d^2}+\frac{ b^2 }{d q} + \frac{(b-d)^2}{d^2} \approx 1$ and thus is independent of the dimension $d$ and the parameter $q$.  

\noindent{\textbf{Adaptive random gradient estimation.}} 
Based on Theorem \ref{thm_grad} and our error analysis,
in AutoZOOM we set $b=q$ in \eqref{eqn_grad_rand} and 
propose to use an adaptive strategy for selecting $q$. AutoZOOM uses $q=1$ (i.e., the fewest possible model evaluation) to first obtain rough gradient estimates for solving \eqref{eqs_attack_optimization_problem} until 
a successful adversarial image is found. After the initial attack success, it switches to use more accurate gradient estimates with $q>1$ to fine-tune the image quality. The trade-off between $q$ (which is proportional to query counts) and distortion reduction will be investigated in Section \ref{sec_perf}.

\subsection{Attack Dimension Reduction via Autoencoder}
\label{subsec_AE}

\noindent{\textbf{Dimension-dependent convergence rate using gradient estimation.}} Different from the first order convergence results, the convergence rate of zeroth order gradient descent methods has an additional multiplicative dimension-dependent factor $d$. In the convex loss setting the rate is $O(\sqrt{d/T})$, where $T$ is the number of iterations \cite{nesterov2017random,liu2017zeroth,gao2014information,wang2017stochastic}. The same convergence rate has also been found in the nonconvex setting \cite{ghadimi2013stochastic}. The dimension-dependent convergence factor $d$  suggests that vanilla black-box attacks using gradient estimations can be query inefficient when the (vectorized) image dimension $d$ is large, due to the curse of dimensionality in convergence. This also motivates us to propose using an autoencoder to reduce the attack dimension and improve query efficiency in black-box attacks. 

In AutoZOOM, we propose to perform random gradient estimation from a reduced dimension $d^\prime < d$ to improve query efficiency. Specifically, as illustrated in Figure \ref{fig:autoencoder}, the additive perturbation to an image $\bx_0$ is actually implemented through a ``decoder'' $D: \mathbb{R}^{d^\prime} \mapsto \mathbb{R}^d $ such that $\bx=\bx_0 + D(\bdelta^\prime)$, where  $\bdelta^\prime \in \mathbb{R}^{d^\prime}$. In other words,  the adversarial perturbation $\bdelta \in \mathbb{R}^d$ to  $\bx_0$ is in fact generated from a dimension-reduced space, with an aim of improving query efficiency due to the reduced dimension-dependent factor in the convergence analysis.
AutoZOOM provides two modes for such a decoder $D$: \\
$\bullet$ An autoencoder (AE) trained on unlabeled data that are different from the training data to learn reconstruction from a dimension-reduced representation. The encoder $E(\cdot)$ in an AE compresses the data to a low-dimensional latent space and the decoder $D(\cdot)$ reconstructs an example from its latent representation. The weights of an AE are learned to minimize the average $L_2$ reconstruction error. Note that training such an AE for black-box adversarial attacks is one-time and is entirely offline (i.e., no model queries needed).  \\
$\bullet$ A simple channel-wise bilinear image resizer (BiLIN) that scales a small image to a large image via bilinear extrapolation\footnote{See \url{tf.image.resize_images}, a TensorFlow example.}. Note that no additional training is required for BiLIN.


\begin{algorithm}[t]
	\caption{AutoZOOM for black-box attacks on DNNs}
	\label{algo_autozoom}
	\begin{algorithmic}
		\State \textbf{Input:} Black-box DNN model $F$, original example $\bx_0$, distortion measure $\Dist(\cdot)$, attack objective $\Loss(\cdot)$, monotonic transformation $M(\cdot)$, decoder $D(\cdot) \in \{\textnormal{AE},\textnormal{BiLIN}\}$, initial coefficient $\lambda_{\textnormal{ini}}$, query budget $Q$
		\While{query count $\leq Q$}        
		\State \textbf{1. Exploration:} use $\bx=\bx_0+D(\bdelta^\prime)$ and apply the random gradient estimator in \eqref{eqn_grad_rand_avg} with
		$q=1$ 	 to the downstream optimizer (e.g., ADAM) for solving \eqref{eqs_attack_optimization_problem} until an initial attack is found.
		\State \textbf{2. Exploitation (post-success stage):}  continue to fine-tune the adversarial perturbation $D(\bdelta^\prime)$ for solving \eqref{eqs_attack_optimization_problem} while setting $q \geq 1$ in \eqref{eqn_grad_rand_avg}.
		\EndWhile        
		\State \textbf{Output:} Least distorted successful adversarial example		
	\end{algorithmic}
\end{algorithm} 

\noindent{\textbf{Why AE?}} Our proposal of AE is motivated by the insightful findings in \cite{goodfellow2014explaining} that a successful adversarial perturbation is highly relevant
to some human-imperceptible noise pattern resembling the shape of the target class, known as the ``shadow''. Since 
a decoder in AE learns to reconstruct data from latent representations, it can also provide distributional guidance for mapping adversarial perturbations to generate these shadows.

We also note that for any reduced dimension $d^\prime$, the setting $b^*=q$ is optimal in terms of minimizing the corresponding estimation error from Theorem \ref{thm_grad}, despite the fact that the gradient estimation errors of different reduced dimensions cannot be directly compared. 
In Section \ref{sec_perf} we will report the superior query efficiency in black-box attacks achieved with the use of AE or BiLIN as the decoder, and discuss the benefit of attack dimension reduction.

\subsection{AutoZOOM Algorithm}
Algorithm \ref{algo_autozoom} summarizes the AutoZOOM framework towards query-efficient black-box attacks on DNNs. We also note that AutoZOOM is a general acceleration tool that is compatible with any gradient-estimation based black-box adversarial attack obeying the attack formulation in \eqref{eqs_attack_optimization_problem}. It also has some theoretical estimation error guarantees and query-efficient parameter selection based on Theorem \ref{thm_grad}.
The details on adjusting the regularization coefficient $\lambda$ and the query parameter $q$ based on run-time model evaluation results will be discussed in Section \ref{sec_perf}.
Our source code is publicly available\footnote{\url{https://github.com/IBM/Autozoom-Attack}}.

\section{Performance Evaluation}
\label{sec_perf}
This section presents the experiments for assessing the performance of AutoZOOM in accelerating black-box attacks on DNNs in terms of the number of queries required for an initial attack success and for a specific distortion level.

\subsection{Distortion Measure and Attack Objective}
As described in Section \ref{sec_auto}, AutoZOOM is a query-efficient gradient-free optimization framework for solving the black-box attack formulation in (\ref{eqs_attack_optimization_problem}). In the following experiments, we demonstrate the utility of AutoZOOM by using the same attack formulation proposed in ZOO \cite{chen2017zoo}, which uses the squared $L_2$ norm as the distortion measure $\Dist(\cdot)$ and adopts the attack objective
\begin{align}
\Loss=\max \{ \max_{j \neq t} \log [F(\bx)]_j -  \log [F(\bx)]_t\},0 \},
\end{align}
where this hinge function is designed for targeted black-box attacks on the DNN model $F$, and the monotonic transformation $M(\cdot) = \log (\cdot)$ is applied to the model output.

\subsection{Comparative Black-box Attack Methods}
We compare \textbf{AutoZOOM-AE} ($D=\textnormal{AE}$) and \textbf{AutoZOOM-BiLIN} ($D=\textnormal{BiLIN}$) with two different baselines: (i) Standard \textbf{ZOO} implementation\footnote{\url{https://github.com/huanzhang12/ZOO-Attack}} with bilinear scaling (same as BiLIN) for dimension reduction; (ii)  \textbf{ZOO+AE}, which is ZOO with AE.
Note that all attacks indeed generate adversarial perturbations based on the same reduced attack dimension.

\subsection{Experiment Setup, Evaluation, Datasets and AutoZOOM Implementation}
We assess the performance of different attack methods on several representative benchmark datasets, including MNIST \cite{lecun1998gradient}, CIFAR-10 \cite{krizhevsky2009learning} and ImageNet \cite{russakovsky2015imagenet}. For MNIST and CIFAR-10, we use the same DNN image classification models\footnote{\url{https://github.com/carlini/nn_robust_attacks}} as in
\cite{carlini2017towards}. For ImageNet, we use the Inception-v3 model \cite{szegedy2016rethinking}.
All experiments were conducted using TensorFlow Machine-Learning Library \cite{abadi2016tensorflow} on machines equipped with an Intel Xeon E5-2690v3 CPU and an Nvidia Tesla K80 GPU.

All attacks used ADAM \cite{kingma2014adam} for solving \eqref{eqs_attack_optimization_problem} with their estimated gradients and the same initial learning rate $2 \times 10^{-3}$. On MNIST and CIFAR-10, all methods adopt 1,000 ADAM iterations. On ImageNet, ZOO and ZOO+AE adopt 20,000 iterations, whereas AutoZOOM-BiLIN and AutoZOOM-AE adopt 100,000 iterations. Note that due to different gradient estimation methods, the query counts (i.e., the number of model evaluations)  per iteration of a black-box attack may vary. ZOO and ZOO+AE use the parallel gradient update of \eqref{eqn_grad_coor} with a batch of $128$ pixels, yielding 256 query counts per iteration. 
AutoZOOM-BiLIN and AutoZOOM-AE use the averaged random full gradient estimator in \eqref{eqn_grad_rand_avg}, resulting in $q+1$ query counts per iteration.
For a fair comparison, the query counts are used for performance assessment.

\noindent{\textbf{Query reduction ratio.}} We use the mean query counts of ZOO with the smallest $\lambda_{\textnormal{ini}}$ as the baseline for computing the query reduction ratio of other methods and configurations.

\noindent{\textbf{TPR and initial success.}} We report the true positive rate (TPR), which measures the percentage of successful attacks fulfilling a pre-defined constraint $\ell$ on the normalized (per-pixel) $L_2$ distortion, as well as their query counts of first successes. We also report the per-pixel $L_2$ distortions of initial successes, where an initial success refers to the first query count that finds a successful adversarial example.


\noindent{\textbf{Post-success fine-tuning.}} When implementing AutoZOOM in Algorithm \ref{algo_autozoom},
on MNIST and CIFAR-10 we find that AutoZOOM without fine-tuning (i.e., $q=1$) already yields similar distortion as ZOO. We note that ZOO can be viewed as coordinate-wise fine-tuning and is thus query-inefficient.
On ImageNet, we will investigate the effect of post-success fine-tuning on reducing distortion.

\noindent{\textbf{Autoencoder Training.}} In AutoZOOM-AE, we use convolutional autoencoders for attack dimension reduction, which are trained on unlabeled datasets that are different from the training dataset and the attacked natural examples. The implementation details are given in the supplementary material.

\noindent{\textbf{Dynamic Switching on $\lambda$.}}
To adjust the regularization coefficient $\lambda$ in \eqref{eqs_attack_optimization_problem}, in all methods
we set its initial value $\lambda_{\textnormal{ini}} \in \{0.1,1,10\}$ on MNIST and CIFAR-10, and set $\lambda_{\textnormal{ini}}=10 $ on ImageNet. Furthermore, for balancing the distortion $\Dist$ and the attack objective $\Loss$ in \eqref{eqs_attack_optimization_problem}, we use a \textit{dynamic switching} strategy to update $\lambda$ during the optimization process. Per every $S$ iterations, $\lambda$ is multiplied by 10 times of the current value if the attack has never been successful. Otherwise, it divides its current value by 2. On MNIST and CIFAR-10, we set $S=100$. On ImageNet, we set $S=1,000$.
At the instance of initial success, we also reset  $\lambda=\lambda_{\textnormal{ini}}$  and the ADAM parameters to the default values, as doing so  can empirically reduce the distortion for all attack methods.

\begin{table*}[t]
	\centering
	\caption{Performance evaluation of black-box targeted attacks on MNIST} 
	\label{table:mnist_mode1}
	\begin{adjustbox}{max width=0.97\textwidth}	
		\begin{tabular}{cccccccc}
			
			\hline
			Method
			& $\lambda_{\textnormal{ini}}$
			& \thead{Attack success\\rate (ASR)}
			& \thead{Mean query count\\(initial success)}
			& \thead{Mean query\\count reduction\\ratio (initial success)}
			& \thead{Mean per-pixel\\$L_{2}$ distortion\\(initial success)}
			& \thead{True positive\\rate (TPR)}
			& \thead{Mean query count\\with per-pixel $L_{2}$\\distortion $\leq$ 0.004}\\
			
			\hline
			& 0.1 & 99.44\% & 35,737.60 &  0.00\% & 3.50$\times10^{-3}$ & 96.76\% & 47,342.85 \\
			ZOO 	 &   1 & 99.44\% & 16,533.30 & 53.74\% & 3.74$\times10^{-3}$ & 97.09\% & 31,322.44 \\
			&  10 & 99.44\% & 13,324.60 & 62.72\% & 4.85$\times10^{-3}$ & 96.31\% & 41,302.12 \\
			\hline
			& 0.1 & 99.67\% & 34,093.95 &  4.60\% & 3.43$\times10^{-3}$ & 97.66\% & 44,079.92 \\
			ZOO+AE   &   1 & 99.78\% & 15,065.52 & 57.84\% & 3.72$\times10^{-3}$ & 98.00\% & 29,213.95 \\
			&  10 & 99.67\% & 12,102.20 & 66.14\% & 4.66$\times10^{-3}$ & 97.66\% & 38,795.98 \\
			\hline
			& 0.1 & 99.89\% &  2,465.95 & 93.10\% & 4.51$\times10^{-3}$ & 96.55\% &  3,941.88 \\
			AutoZOOM-BiLIN   &   1 & 99.89\% &    879.98 & 97.54\% & 4.12$\times10^{-3}$ & 97.89\% &  2,320.01 \\
			&  10 & 99.89\% &    612.34 & 98.29\% & 4.67$\times10^{-3}$ & 97.11\% &  4,729.12 \\
			\hline
			& 0.1 &\textbf{100.00\%} &  2,428.24 & \textbf{93.21\%} & 4.54$\times10^{-3}$ & 96.67\% &  3,861.30 \\
			AutoZOOM-AE &   1 &\textbf{100.00\%} &    729.65 & \textbf{97.96\%} & 4.13$\times10^{-3}$ & 96.89\% &  1,971.26 \\
			&  10 &\textbf{100.00\%} &    510.38 & \textbf{98.57\%} & 4.67$\times10^{-3}$ & 97.22\% &  4,855.01 \\
			
			\hline
		\end{tabular}
	\end{adjustbox}
\end{table*}


\begin{table*}[t]
	\centering
	\caption{Performance evaluation of black-box targeted attacks on CIFAR-10} 
	\label{table:cifar_mode_2}
	\begin{adjustbox}{max width=0.97\textwidth}	
		\begin{tabular}{cccccccc}
			
			\hline
			Method
			& $\lambda_{\textnormal{ini}}$
			& \thead{Attack success\\rate (ASR)}
			& \thead{Mean query count\\(initial success)}
			& \thead{Mean query\\count reduction\\ratio (initial success)}
			& \thead{Mean per-pixel\\$L_{2}$ distortion\\(initial success)}
			& \thead{True positive\\rate (TPR)}
			& \thead{Mean query count\\with per-pixel $L_{2}$\\distortion $\leq$ 0.0015}\\
			
			\hline
			& 0.1 & 97.00\% & 25,538.43 &  0.00\% & 5.42$\times10^{-4}$ &100.00\% & 25,568.33 \\
			ZOO 	 &   1 & 97.00\% & 11,662.80 & 54.33\% & 6.37$\times10^{-4}$ &100.00\% & 11,777.18 \\
			&  10 & 97.00\% & 10,015.08 & 60.78\% & 8.03$\times10^{-4}$ &100.00\% & 10,784.54 \\
			\hline
			& 0.1 & 99.33\% & 19,670.96 & 22.98\% & 4.96$\times10^{-4}$ &100.00\% & 20,219.42 \\
			ZOO+AE   &   1 & 99.00\% &  5,793.25 & 77.32\% & 6.83$\times10^{-4}$ & 99.89\% &  5,773.24 \\
			&  10 & 99.00\% &  4,892.80 & 80.84\% & 8.74$\times10^{-4}$ & 99.78\% &  5,378.30 \\
			\hline
			& 0.1 & 99.67\% &  2,049.28 & 91.98\% & 1.01$\times10^{-3}$ & 98.77\% &  2,112.52 \\
			AutoZOOM-BiLIN   &   1 & 99.67\% &    813.01 & 96.82\% & 8.25$\times10^{-4}$ & 99.22\% &  1,005.92 \\
			&  10 & 99.33\% &    623.96 & 97.56\% & 9.09$\times10^{-4}$ & 98.99\% &    835.27 \\
			\hline
			& 0.1 &\textbf{100.00\%} &  1,523.91 & \textbf{94.03\%} & 1.20$\times10^{-3}$ & 99.67\% &  1,752.45 \\
			AutoZOOM-AE &   1 &\textbf{100.00\%} &    332.43 & \textbf{98.70\%} & 1.01$\times10^{-3}$ & 99.56\% &    345.62 \\
			&  10 &\textbf{100.00\%} &    259.34 & \textbf{98.98\%} & 1.15$\times10^{-3}$ & 99.67\% &    990.61 \\
			
			\hline
		\end{tabular}
	\end{adjustbox}
\end{table*}

\subsection{Black-box Attacks on MNIST and CIFAR-10}
For both MNIST and CIFAR-10, we randomly select 50 correctly classified images from their test sets, and perform targeted attacks on these images. Since both datasets have 10 classes, each selected image is attacked 9 times, targeting at all but its true class. For all attacks, the ratio of reduced attack-space dimension to the original one (i.e., $d^\prime/d$) is 25\% for MNIST and 6.25\% for CIFAR-10. 

Table \ref{table:mnist_mode1} shows the performance evaluation on MNIST with various values of $\lambda_{\textnormal{ini}}$, the initial value of the regularization coefficient $\lambda$ in (\ref{eqs_attack_optimization_problem}). We use the performance of ZOO with $\lambda_{\textnormal{ini}} = 0.1$ as a baseline for comparison.   
For example, with $\lambda_{\textnormal{ini}} =$ $0.1$ and $10$, the mean query counts required by AutoZOOM-AE to attain an initial success is reduced by \textbf{93.21\%} and \textbf{98.57\%}, respectively. One can also observe that allowing larger $\lambda_{\textnormal{ini}}$ generally leads to fewer mean query counts at the price of slightly increased distortion for the initial attack. 
The noticeable huge difference in the required attack query counts between AutoZOOM and ZOO/ZOO+AE validates the effectiveness of our proposed random full gradient estimator in \eqref{eqn_grad_rand}, which dispenses with the  coordinate-wise gradient estimation in ZOO but still remains comparable true positive rates, 
thereby greatly improving query efficiency.


For CIFAR-10, we report similar query efficiency improvements as displayed in Table \ref{table:cifar_mode_2}. In particular, comparing the two query-efficient black-box attack methods (AutoZOOM-BiLIN and AutoZOOM-AE), we find that AutoZOOM-AE is more query-efficient than AutoZOOM-BiLIN, but at the cost of an additional AE training step. AutoZOOM-AE achieves the highest attack success rates (ASRs) and mean query reduction ratios for different values of $\lambda_{\textnormal{ini}}$. In addition, their true positive rates (TPRs) are similar but AutoZOOM-AE usually takes fewer query counts to reach the same $L_2$ distortion. We note that when $\lambda_{\textnormal{ini}}=10$, AutoZOOM-AE has a higher TPR but also needs slightly more mean query counts than AutoZOOM-BiLIN to reach the same $L_2$ distortion. This suggests that there are some adversarial examples that are difficult for a bilinear resizer to reduce their post-success distortions but can be handled by an AE.


\begin{table*}[t]
	\centering
	\caption{Performance evaluation of black-box targeted attacks on ImageNet} 
	\label{table:ImageNet50}
	\begin{adjustbox}{max width=0.97\textwidth}	
		\begin{tabular}{cccccccc}
			
			\hline
			Method
			& \thead{Attack success\\rate (ASR)}
			& \thead{Mean query count\\(initial success)}
			& \thead{Mean query\\count reduction\\ratio (initial success)}
			& \thead{Mean per-pixel\\$L_{2}$ distortion\\(initial success)}
			& \thead{True positive\\rate (TPR)}
			& \thead{Mean query count\\with per-pixel $L_{2}$\\distortion $\leq$ 0.0002}\\
			
			\hline
			ZOO 	 & 76.00\% & 2,226,405.04 (2.22M) &  0.00\% & 4.25$\times10^{-5}$ & 100.00\% & 2,296,293.73 \\
			\hline
			ZOO+AE   & 92.00\% & 1,588,919.65 (1.58M) & 28.63\% & 1.72$\times10^{-4}$ & 100.00\% & 1,613,078.27 \\
			\hline
			AutoZOOM-BiLIN   &\textbf{100.00\%} &    14,228.88 & 99.36\% & 1.26$\times10^{-4}$ & 100.00\% &    15,064.00 \\
			\hline
			AutoZOOM-AE &\textbf{100.00\%} &    13,525.00 & \textbf{99.39\%} & 1.36$\times10^{-4}$ & 100.00\% &    14,914.92 \\
			
			\hline
		\end{tabular}
	\end{adjustbox}
	\vspace{-4mm}
\end{table*}

\begin{figure*}[t]
	\centering
	\begin{subfigure}[b]{0.4\textwidth}
		\includegraphics[width=\textwidth]{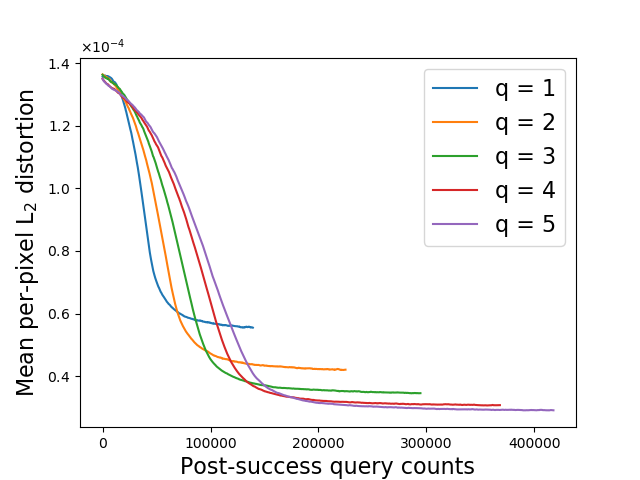}
		\caption{Post-success distortion refinement}
	\end{subfigure}%
	\hspace{3mm}
	\centering
	\begin{subfigure}[b]{0.38\textwidth}
		\includegraphics[width=\textwidth]{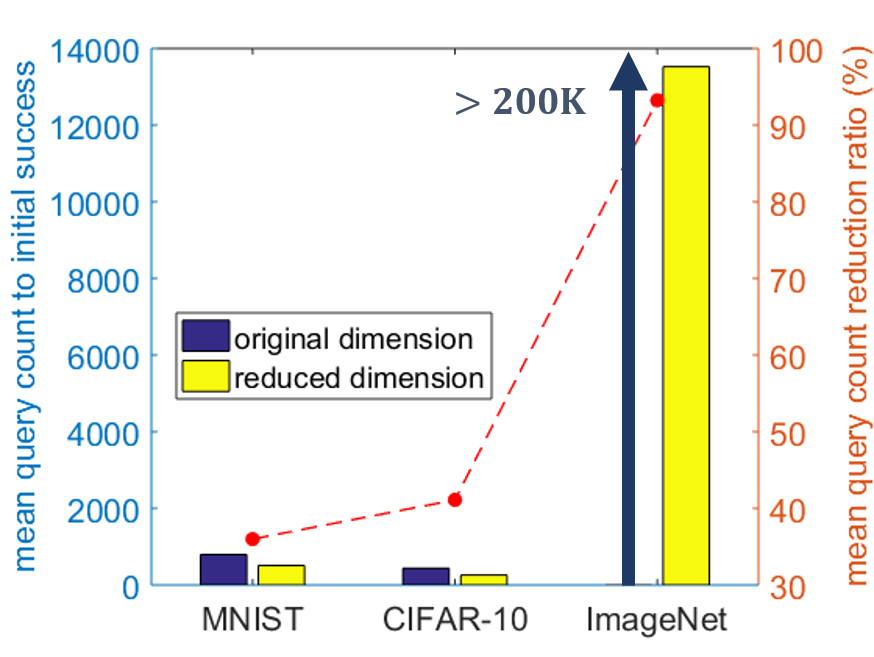}
		\caption{Dimension reduction v.s. query efficiency}
	\end{subfigure}
	\caption{(a) After initial success, 
		AutoZOOM (here $D=\textnormal{AE}$) can further decrease the distortion by setting $q>1$ in \eqref{eqn_grad_rand_avg} to trade more query counts for smaller distortion in the converged stage, which saturates at $q=4$. (b) Attack dimension reduction is crucial to query-efficient black-box attacks. When compared to black-box attacks on the original dimension, dimension reduction through AutoZOOM-AE reduces roughly 35-40\% query counts on MNIST and CIFAR-10 and at least 95\% on ImageNet.
	}
	\label{Fig_compare}
\end{figure*}

\subsection{Black-box Attacks on ImageNet}
\label{subsec_imagenet}
We selected 50 correctly classified images from the ImageNet test set to perform random targeted attacks and set $\lambda_{\textnormal{ini}}=10$ and the attack dimension reduction ratio to 1.15\%. The results are summarized in Table \ref{table:ImageNet50}. Note that comparing to ZOO, AutoZOOM-AE can significantly reduce the query count required to achieve an initial success by 99.39\% (or 99.35\% to reach the same $L_2$ distortion), which is a remarkable improvement since this means reducing more than \textit{2.2 million} model queries given the fact that the dimension of ImageNet ($\approx$ 270K) is much larger than that of MNIST and CIFAR-10.

\noindent{\textbf{Post-success distortion refinement.}}
As described in Algorithm \ref{algo_autozoom}, adaptive random gradient estimation is integrated in AutoZOOM, offering a quick initial success in attack generation followed by a fine-tuning process to effectively reduce the distortion. This is achieved by adjusting the gradient estimate averaging parameter $q$ in (\ref{eqn_grad_rand_avg}) in the post-success stage. In general, averaging over more random directions (i.e., setting larger $q$) tends to better reduce the variance of gradient estimation error, but at the cost of increased model queries. Figure \ref{Fig_compare} (a) shows the mean distortion against query counts for various choices of $q$ in the post-success stage. The results suggest that setting some small $q$ but $q>1$ can further decrease the distortion at the converged phase when compared with the case of $q=1$. Moreover, the refinement effect on distortion empirically saturates at $q=4$, implying a marginal gain beyond this value. These findings also demonstrate that our proposed AutoZOOM indeed strikes a balance between distortion and query efficiency in black-box attacks.



\subsection{Dimension Reduction and Query Efficiency}
In addition to the motivation from the $O(\sqrt{d/T})$ convergence rate in zeroth-order optimization (Sec. \ref{subsec_AE}), as a sanity check,  we corroborate the benefit of attack dimension reduction to query efficiency in black-box attacks by comparing AutoZOOM (here we use $D=\textnormal{AE}$) with its alternative operated on the original (non-reduced) dimension (i.e., $\delta^\prime=D(\delta^\prime)=\delta$). 
Tested on all three datasets and aforementioned settings, Figure \ref{Fig_compare} (b) shows the corresponding mean query count to initial success and the mean query reduction ratio when $\lambda_{\textnormal{ini}}=10$ in all three datasets.  When compared to the attack results of the original dimension, attack dimension reduction through AutoZOOM reduces roughly 35-40\% query counts on MNIST and CIFAR-10 and at least 95\% on ImageNet. This result highlights the importance of dimension reduction towards query-efficient black-box attacks. For example, without dimension reduction, the attack on the original ImageNet dimension cannot even be successful within the query budge ($Q=200K$ queries).

\subsection{Additional Remarks and Discussion}
$\bullet$ In addition to benchmarking on initial attack success, the query reduction ratio when reaching the same $L_2$ distortion can be directly computed from the last column in each table. \\
$\bullet$ The attack gain in AutoZOOM-AE versus AutoZOOM-BiLIN could sometimes be marginal, while we also note that there is room for improving AutoZOOM-AE by exploring different AE models. However, we advocate AutoZOOM-BiLIN as a practically ideal candidate for query-efficient black-box attacks when testing model robustness, due to its easy-to-mount nature and it has no additional training cost. \\
$\bullet$ While learning effective low-dimensional representations of legitimate images is still a challenging task, black-box attacks using significantly less degree of freedoms (i.e., reduced dimensions), as demonstrated in this paper, are certainly plausible, leading to new implications on model robustness.




\section{Conclusion}

AutoZOOM is a generic attack acceleration framework that is compatible with any gradient-estimation based black-box attack having the general formulation in \eqref{eqs_attack_optimization_problem}. It adopts a new and adaptive random full gradient estimation strategy to strike a balance between query counts and estimation errors, and features
a decoder (AE or BiLIN) for attack dimension reduction and algorithmic convergence acceleration.  Compared to a state-of-the-art attack (ZOO), AutoZOOM consistently reduces the mean query counts when attacking black-box DNN image classifiers for MNIST, CIFAT-10 and ImageNet, attaining at least $93\%$ query reduction in finding initial successful adversarial examples (or reaching the same distortion) while maintaining a similar attack success rate. It can also efficiently fine-tune the image distortion to maintain high visual similarity to the original image.
Consequently, AutoZOOM provides novel and efficient means for assessing the robustness of deployed machine learning models.

%
%
%
%

\section*{Acknowledgements}
Shin-Ming Cheng was supported in part by the Ministry of Science and Technology,
Taiwan, under Grants MOST 107-2218-E-001-005 and MOST 107-2218-E-011-012. Cho-Jui Hsieh and Huan Zhang acknowledge the support by NSF IIS-1719097, Intel faculty award, Google Cloud  and  NVIDIA.

{\small
\bibliographystyle{aaai}
\bibliography{adversarial_learning}
}

\clearpage
\setcounter{equation}{0}
\setcounter{figure}{0}
\setcounter{table}{0}
\setcounter{page}{1}
\makeatletter
\renewcommand{\theequation}{S\arabic{equation}}
\renewcommand{\thefigure}{S\arabic{figure}}
\renewcommand{\thetable}{S\arabic{table}}
\section*{{ \LARGE Supplementary Material}}
\appendix

\section{More Background on Adversarial Attacks and Defenses}
The research in generating adversarial examples to deceive machine-learning models, known as adversarial attacks, tends to evolve with the advance of machine-learning techniques and new publicly available datasets. In \cite{lowd2005adversarial}, the authors studied adversarial attacks to linear classifiers with continuous or Boolean features. In \cite{biggio2013evasion}, the authors proposed a gradient-based adversarial attack on kernel support vector machines (SVMs). More recently, gradient-based approaches are also used in adversarial attacks on image classifiers trained by DNNs  \cite{szegedy2013intriguing,goodfellow2014explaining}. Due to space limitation, we focus on related work in adversarial attacks on DNNs. 
Interested readers may refer to the survey paper \cite{biggio2018wild} for more details.

\begin{table*}[t]
	\centering
	\caption{Comparison of black-box attacks on DNNs}
	\label{table_black_box}
	\begin{adjustbox}{max width=\textwidth}
		\begin{tabular}{|c|c|c|c|c|c|}
			\hline
			Method                         & Approach                                                                       & \begin{tabular}[c]{@{}c@{}}Model\\ ouput\end{tabular} & \begin{tabular}[c]{@{}c@{}}Targeted \\ attack\end{tabular} & \begin{tabular}[c]{@{}c@{}}Large network\\ (ImageNet)\end{tabular} & \begin{tabular}[c]{@{}c@{}}Data-driven \\  acceleration\end{tabular} \\ \hline
			\cite{narodytska2016simple}  & local random search                                                            & score                                                 &                                                            & $\checkmark$                                                       &                                                                               \\ \hline
			\cite{papernot2017practical} & substitute model                                                               & score                                                 & $\checkmark$                                               &                                                                    &                                                                               \\ \hline
			\cite{suya2017query}         & acquisition via posterior                                                      & score                                                 & $\checkmark$                                               &                                                                    &                                                                               \\ \hline
			\cite{brendel2017decision}   & Gaussian perturbation                                                          & decision                                              & $\checkmark$                                               & $\checkmark$                                                       &                                                                               \\ \hline
			\cite{ilyas2018black}        & natural evolution strategy                                                     & score/decision                                                 & $\checkmark$                                               & $\checkmark$                                                       &                                                                               \\ \hline
			\cite{chen2017zoo}           & \begin{tabular}[c]{@{}c@{}}coordinate-wise \\ gradient estimation\end{tabular} & score                                                 & $\checkmark$                                               & $\checkmark$                                                       &                                                                               \\ \hline
			\cite{bhagoji2017exploring}  & \begin{tabular}[c]{@{}c@{}}coordinate-wise \\ gradient estimation\end{tabular} & score                                                 & $\checkmark$                                               & $\checkmark$                                                       &                                                                               \\ \hline
			AutoZOOM (this paper)          & \begin{tabular}[c]{@{}c@{}}Random (full)\\ gradient estimation\end{tabular}    & score                                                 & $\checkmark$                                               & $\checkmark$                                                       & $\checkmark$                                                                  \\ \hline
		\end{tabular}
	\end{adjustbox}
\end{table*}

Gradient-based adversarial attacks on DNNs fall within the white-box setting, since acquiring the gradient with respect to the input requires knowing the weights of the target DNN. In principle, adversarial attacks can be formulated as an optimization problem of minimizing the adversarial perturbation while ensuring attack objectives. In image classification, given a natural image, an \textit{untargeted} attack aims to find a visually similar adversarial image resulting in a different class prediction, while a \textit{targeted} attack aims to find an adversarial image leading to a specific class prediction. The visual similarity between a pair of adversarial and natural images is often measured by the $L_p$ norm of their difference, where $p \geq 1$. Existing powerful white-box adversarial attacks using $L_{\infty}$, $L_2$ or $L_1$ norms include iterative fast gradient sign methods \cite{kurakin2016adversarial_ICLR}, Carlini and Wagner's (C\&W) attack  \cite{carlini2017towards}, elastic-net attacks to DNNs (EAD) \cite{chen2017ead}, etc.  

Black-box adversarial attacks are practical threats to the deployed machine-learning services. Attackers can observe the input-output correspondences of any queried input, but the target model parameters are completely hidden. Therefore, gradient-based adversarial attacks are inapplicable to a black-box setting. As a first attempt, the authors in \cite{papernot2017practical} proposed to train a substitute model using iterative model queries, perform white-box attacks on the substitute model, and leverage the transferability of adversarial examples \cite{papernot2016transferability,liu2016delving} to attack the target model. However, training a representative surrogate for a DNN is challenging due to the complicated and nonlinear classification rules of DNNs and high dimensionality of the underlying dataset. The performance of black-box attacks can be severely degraded if the adversarial examples for the substitute model transfer poorly to the target model.
To bridge this gap, the authors in \cite{chen2017zoo} proposed a black-box attack called ZOO that directly estimates the gradient of the attack objective by iteratively querying the target model. Although ZOO achieves a similar attack success rate and comparable visual quality as many white-box attack methods, it exploits the symmetric difference quotient method \cite{lax2014calculus} for coordinate-wise gradient estimation and value update, which requires excessive target model evaluations and is hence not query-efficient. The same gradient estimation technique is also used in the later work in \cite{bhagoji2017exploring}. Although acceleration techniques such as importance sampling, bilinear scaling and random feature grouping have been used in \cite{chen2017zoo,bhagoji2017exploring}, the coordinate-wise gradient estimation approach still forms a bottleneck for query efficiency.

Beyond optimization-based approaches, the authors in \cite{ilyas2018black} proposed to use a natural evolution strategy (NES) to enhance  query efficiency. Although there is also a vector-wise gradient estimation step in the NES attack, we treat it as an independent and parallel work since its natural evolutionary step is out of the scope of black-box attacks using zeroth-order gradient descent.  We also note that different from NES, our AutoZOOM framework uses a query-efficient random gradient estimation strategy. In addition, AutoZOOM
could be applied to further improve the query efficiency of NES, since NES does not take into account the factor of attack dimension reduction, which is the main focus of this paper.
Under a more restricted setting, where only the decision (top-1 prediction class) is known to an attacker,
the authors in \cite{brendel2017decision} proposed a random-walk based attack around the decision boundary.
Such a black-box attack dispenses class prediction scores and hence requires additional model queries.

In this paper, we focus on improving the query efficiency of gradient-estimation and gradient-descent based black-box attacks and consider the threat model when the class prediction scores are known to an attacker. For reader's reference, we compare existing black-box attacks on DNNs with AutoZOOM in Table \ref{table_black_box}. One unique feature of AutoZOOM 
is the use of reduced attack dimension when mounting black-box attacks, which is an unlabeled data-driven technique (autoencoder) for attack acceleration, and has not been studied thoroughly in existing attacks. While white-box attacks such as \cite{baluja2017adversarial} have utilized autoencoders trained on the training data and the transparent logit representations of DNNs, we propose in this work to use autoencoders trained on unlabeled natural data to improve query efficiency for black-box attacks.   

There has been many methods proposed for defending adversarial attacks to DNNs. However, new defenses are continuously weakened by follow-up attacks \cite{carlini2017adversarial,athalye2018obfuscated}. For instance, model ensembles \cite{tramer2017ensemble} were shown to be effective against some black-box attacks, while they are recently circumvented by advanced attack techniques \cite{ilyas2018ensattack}. In this paper, we focus on improving query efficiency in attacking black-box undefended DNNs.

\section{Proof of Theorem \ref{thm_grad}}
Recall that the data dimension is $d$ and we assume $f$ to be differentiable and its gradient $\nabla f$ to be $L$-Lipschitz.
Fixing $\beta$ and consider  a smoothed version of $f$:
\begin{align}
{f}_\beta(\mathbf x) = & \mathbb E_{\mathbf u}[f(\mathbf x + \beta \mathbf u )].
\end{align}
Based on \cite[Lemma 4.1-a]{gao2014information}, we have the relation
\begin{align}
\nabla f_\beta (\mathbf x) &   = \mathbb E_{\mathbf u }
\left [
\frac{d}{\beta} f(\mathbf x + \beta \mathbf u) \mathbf u
\right ]  =    \frac{d}{b} \mathbb E_{\mathbf u }
\left [
{\bg}
\right ],
\end{align}
which then yields
\begin{align}\label{eq: mean_est}
\mathbb E_{\mathbf u }
\left [
\mathbf g
\right ] = \frac{b}{d} \nabla f_\beta (\mathbf x),
\end{align}
where we recall that $\mathbf g$ has been defined in \eqref{eqn_grad_rand}.
Moreover, based on  \cite[Lemma 4.1-b]{gao2014information}, we have
\begin{align}
\label{eqn_dummy_1}
\|   \nabla f_\beta (\mathbf x) - \nabla f(\mathbf x) \|_2 \leq \frac{\beta d L}{2}.
\end{align}
Substituting \eqref{eq: mean_est} into \eqref{eqn_dummy_1}, we obtain
\[
\|  \mathbb E[ \bg ]- \frac{b}{d}\nabla f(\mathbf x) \|_2 \leq \frac{\beta b L}{2}.
\]
This then implies that
\begin{align}\label{eq: mean_batch}
\mathbb E [ \bg ] = \frac{b}{d} \nabla f(\mathbf x) + \boldsymbol{\epsilon},
\end{align}
where 
$
\| \boldsymbol \epsilon \|_2 \leq \frac{b \beta L}{2}.
$

Once again, by applying \cite[Lemma 4.1-b]{gao2014information}, we can easily obtain that
\begin{align}\label{eq: var_est}
\mathbb E_{\mathbf u} [ \| \bg \|_2^2] & 
\leq
\frac{b^2 L^2 \beta^2}{2} + \frac{2 b^2}{d} \| \nabla  f(\mathbf x) \|_2^2.
\end{align}

Now, let us consider the averaged random gradient estimator in \eqref{eqn_grad_rand_avg},
\[
\overline{\mathbf g} = \frac{1}{q}\sum_{i=1}^q {\mathbf g}_i=\frac{b}{q}\sum_{i=1}^q
\frac{f(\mathbf x + \beta \mathbf u_{i}) - f(\mathbf x) }{\beta} \mathbf u_{i}.
\]
Due to the properties of i.i.d. samples   $\{ \mathbf u_{i} \}$ and \eqref{eq: mean_batch}, we define
\begin{align} \label{eq: mean_g_minibatch}
\bv =: \mathbb E [ {\mathbf g}_{i}] =   \frac{b}{d} \nabla f(\mathbf x) + \boldsymbol{\epsilon}.
\end{align}
Moreover, we have
\begin{align}
\mathbb E[ \| \overline {\mathbf g} \|_2^2 ]  =& \mathbb E \left [ \left  \| \frac{1}{q}\sum_{i=1}^{q}  ( {\mathbf g}_{i}  - \bv) + \bv \right \|_2^2 \right ] \\
=& \| \bv \|_2^2 + \mathbb E \left [ \left  \| \frac{1}{q }\sum_{i=1}^{q}   ( {\mathbf g}_{i}  - \bv) \right \|_2^2 \right ] \nonumber \\
= &  \| \bv \|_2^2 + \frac{1}{q} 
\mathbb E [ \left  \|  {\mathbf g}_{1}  - \bv \right \|_2^2  ] \\
= & \| \bv \|_2^2 + \frac{1}{q} \mathbb E[ \| {\mathbf g}_{1} \|_2^2 ]  - \frac{1}{q}  \| \bv \|_2^2,
\label{eq: appF_2_new}
\end{align}
where we have used the fact that $ \mathbb E [ {\mathbf g}_{i}] =  \mathbb E [ {\mathbf g}_{1}] = \bv~\forall~i$.
The definition of $\bv$ in \eqref{eq: mean_g_minibatch} yields
\begin{align}
\| \bv \|_2^2  \leq & 2 \frac{b^2}{d^2} \| \nabla f(\mathbf x) \|_2^2 + 2 \| \boldsymbol \epsilon \|_2^2 \nonumber \\
\leq &  2 \frac{b^2}{d^2} \| \nabla f(\mathbf x) \|_2^2  + \frac{1}{2} b^2 \beta^2 L^2. \label{eq: appF_3_new}
\end{align}
From \eqref{eq: var_est}, we also obtain that for any $i$,
\begin{align}
\mathbb E[ \| {\mathbf g}_{i} \|_2^2 ] \leq \frac{b^2 L^2 \beta^2}{2} + \frac{2 b^2}{d} \| \nabla  f(\mathbf x) \|_2^2. \label{eq: appF_4_new}
\end{align}
Substituting \eqref{eq: appF_3_new} and \eqref{eq: appF_4_new} into \eqref{eq: appF_2_new}, we obtain
\begin{align}
\mathbb E[ \| \overline {\mathbf g} \|_2^2 ] \leq &    \| \bv \|_2^2    +  \frac{1}{q}  
\mathbb E [ \left  \| {\mathbf g}_{1}   \right \|_2^2  ]  \\
\leq &  2 ( \frac{b^2}{d^2}+\frac{ b^2 }{d q} ) \| \nabla f(\mathbf x) \|_2^2 + \frac{q+1}{2q} b^2 L^2 \beta^2. \label{eq: appF_5_new}
\end{align}
Finally, we bound the mean squared estimation error as
\begin{align}\label{eq: mse_2_new}
\mathbb E[ \| \overline {\mathbf g} - \nabla f(\mathbf x)\|_2^2 ] &  \leq  2 \mathbb E [\| \overline {\mathbf g} - \bv \|_2^2  ] + 2 \|   \bv - \nabla f(\mathbf x)\|_2^2 \nonumber \\
& \leq 2 \mathbb E[ \|  \overline {\mathbf g}\|_2^2] + 2 \|   \frac{b}{d} \nabla f(\mathbf x) + \boldsymbol{\epsilon} - \nabla f (\mathbf x)\|_2^2 \nonumber \\
& \leq    4 ( \frac{b^2}{d^2}+\frac{ b^2 }{d q} + \frac{(b-d)^2}{d^2}) \| \nabla f(\mathbf x) \|_2^2 \nonumber\\
&~~+ \frac{2q+1}{q} b^2 L^2 \beta^2,
\end{align}
which completes the proof.





\begin{table*}[t]
	\centering
	\caption{Architectures of Autoencoders in AutoZOOM}
	\label{table:ae_architecture}
	\begin{adjustbox}{max width=\textwidth}
		\begin{tabular}{@{}llllll@{}}
			\toprule
			Dataset: & \multicolumn{4}{l}{MNIST \hspace{3cm}     Training MSE: 2.00$\times10^{-3}$}\\
			&Reduction ratio / image size / feature map size:&25\% / 28$\times$28$\times$1 / 14$\times$14$\times$1\\
			Encoder: & \multicolumn{5}{l}{ConvReLU-16 $\rightarrow$  MaxPool $\rightarrow$ Conv-1} \\
			Decoder: & \multicolumn{5}{l}{ConvReLU-16 $\rightarrow$ Reshape-Re-U $\rightarrow$ Conv-1} \\ \cmidrule(l){2-6} 
			\multicolumn{6}{l}{} \\ \cmidrule(l){2-6}
			
			Dataset:& \multicolumn{4}{l}{CIFAR-10 \hspace{3cm}     Training MSE: 5.00$\times10^{-3}$}\\
			&Reduction ratio / image size / feature map size:&6.25\% / 32$\times$32$\times$3 / 8$\times$8$\times$3\\
			Encoder: & \multicolumn{5}{l}{ConvReLU-16 $\rightarrow$ MaxPool $\rightarrow$ ConvReLU-3 $\rightarrow$ MaxPool $\rightarrow$ Conv-3} \\
			Decoder: & \multicolumn{5}{l}{ConvReLU-16 $\rightarrow$ Reshape-Re-U $\rightarrow$ ConvReLU-16 $\rightarrow$ Reshape-Re-U  $\rightarrow$ Conv-3} \\ \cmidrule(l){2-6} 
			\multicolumn{6}{l}{} \\ \cmidrule(l){2-6}
			
			Dataset:& \multicolumn{4}{l}{ImageNet \hspace{3cm}     Training MSE: 1.02$\times10^{-2}$}\\
			&Reduction ratio / image size / feature map size:&1.15\% / 299$\times$299$\times$3 / 32$\times$32$\times$3\\
			Encoder: & \multicolumn{5}{l}{\begin{tabular}[c]{@{}l@{}}Reshape-Bi-D $\rightarrow$  ConvReLU-16$\rightarrow$ MaxPool $\rightarrow$ ConvReLU-16$\rightarrow$ MaxPool $\rightarrow$ Conv-3\end{tabular}} \\
			Decoder: & \multicolumn{5}{l}{\begin{tabular}[c]{@{}l@{}}ConvReLU-16 $\rightarrow$ Reshape-Re-U $\rightarrow$ ConvReLU-16 $\rightarrow$ Reshape-Bi-U  $\rightarrow$ Conv-3\end{tabular}} \\ \cmidrule(l){2-6}
			
			\multicolumn{6}{l}{} \\ \midrule
			\multicolumn{6}{l}{ConvReLU-16:        Convolution (16 filters, kernel size: 3$\times$3$\times$$Dep$) + ReLU activation} \\
			\multicolumn{6}{l}{ConvReLU-3: Convolution (3 filters, kernel size: 3$\times$3$\times$$Dep$) + ReLU activation} \\
			\multicolumn{6}{l}{Conv-3:        Convolution (3 filters, kernel size: 3$\times$3$\times$$Dep$) \ \ \ \ \ \ \  Conv-1:        Convolution (1 filter, kernel size: 3$\times$3$\times$$Dep$)} \\
			\multicolumn{6}{l}{Reshape-Bi-D: Bilinear reshaping from 299$\times$299$\times$3 to 128$\times$128$\times$3} \\
			\multicolumn{6}{l}{Reshape-Bi-U:        Bilinear reshaping from 128$\times$128$\times$16 to 299$\times$299$\times$3} \\
			\multicolumn{6}{l}{Reshape-Re-U:        Reshaping by replicating pixels from $U\times V\times$$Dep$ to $2U\times2V\times$$Dep$} \\
			\multicolumn{6}{l}{$Dep$: a proper depth} \\ \bottomrule
		\end{tabular}
	\end{adjustbox}
\end{table*}

\section{Architectures of Convolutional Autoencoders in AutoZOOM}

On MNIST, the convolutional autoencoder (CAE) is trained on 50,000 randomly selected hand-written digits from the MNIST8M dataset\footnote{\url{http://leon.bottou.org/projects/infimnist}}. On CIFAR-10, the CAE is trained on 9,900  images selected from its test dataset. The remaining images are used in black-box attacks. On ImageNet, all the attacked natural images are from 10  randomly selected image labels, and these labels are also used as the candidate attack targets. The CAE is trained on about 9000 images from these classes. 

Table \ref{table:ae_architecture} shows the architectures for all the autoencoders used in this work. Note that the autoencoders designed for ImageNet uses bilinear scaling to transform data size from $299\times299\times Dep$ to $128\times128\times Dep$, and also back from $128\times128\times Dep$ to $299\times299\times Dep$. This is to allow easy processing and handling for the autoencoder's internal convolutional layers.

The normalized mean squared error of our autoencoder trained on MNIST, CIFAR-10 and
25 Imagenet is 0.0027, 0.0049 and 0.0151, respectively, which lies within a reasonable range of compression loss.


\section{More Adversarial Examples of Attacking Inception-v3 in the Black-box Setting}
Figure \ref{Fig_more} shows other adversarial examples of the Inception-v3 model in the black-box targeted attack setting.

\begin{figure*}[t]
	\centering   
	\begin{subfigure}[b]{0.48\linewidth}
		\includegraphics[width=\textwidth]{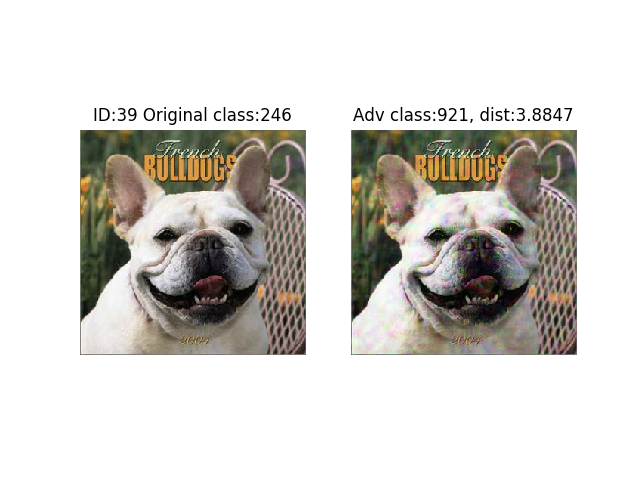}
		\vspace{-15mm}
		\caption{``French bulldog'' to ``traffic light''}
	\end{subfigure}%
	\centering
	\begin{subfigure}[b]{0.48\linewidth}
		\includegraphics[width=\textwidth]{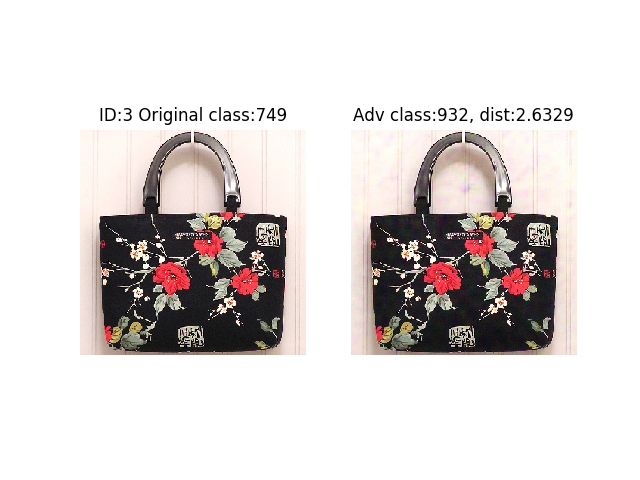}
		\vspace{-15mm}
		\caption{``purse'' to ``bagel''}
	\end{subfigure}
	\\
	\centering
	\begin{subfigure}[b]{0.48\linewidth}
		\includegraphics[width=\textwidth]{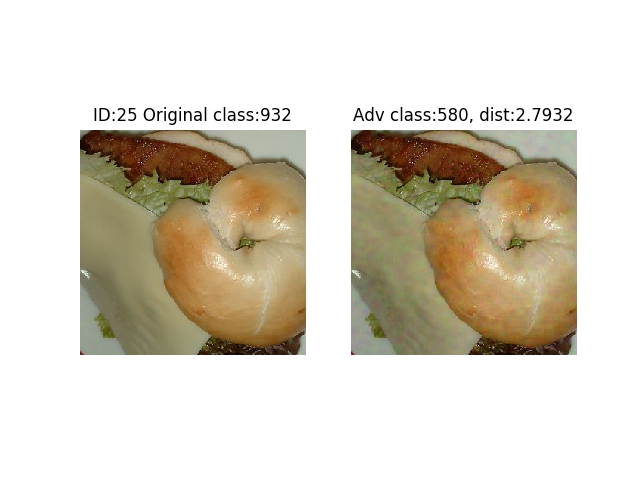}
		\vspace{-15mm}
		\caption{``bagel'' to `` grand piano''}
	\end{subfigure}    
	\centering
	\begin{subfigure}[b]{0.48\linewidth}
		\includegraphics[width=\textwidth]{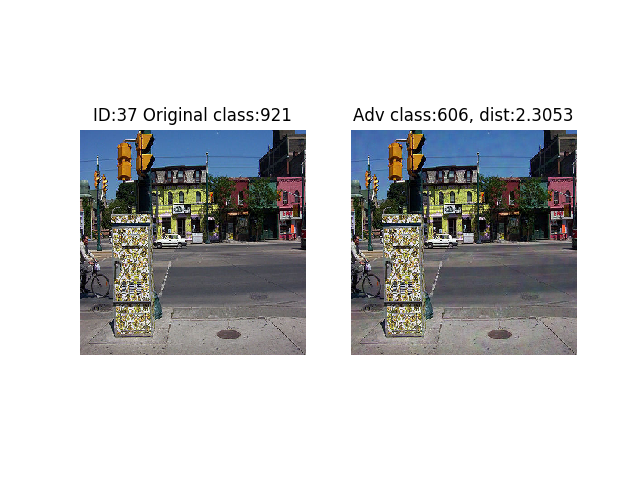}
		\vspace{-15mm}
		\caption{``traffic light'' to `` iPod''}
	\end{subfigure}       
	\caption{Adversarial examples on ImageNet crafted by AutoZOOM when attacking on the Inception-v3 model in the black-box setting with a target class selected at random. 
		Left: original natural images. Right: adversarial examples. }
	\label{Fig_more}
\end{figure*}

\section{Performance Evaluation of Black-box Untargeted Attacks}
Table \ref{table:untargeted} shows the attacking performance of black-box untargeted attacks on MNIST, CIFAR-10 and ImageNet using ZOO and AutoZOOM-BiLIN attacks on the same set of images in Section \ref{subsec_imagenet}. The $\Loss$  function is defined as 
\begin{align}
\Loss=\max \{  \log [F(\bx)]_{t_0} - \max_{j \neq t_0} \log [F(\bx)]_j \},0 \},
\end{align}
where $t_0$ is the top-1 prediction label of a natural image $\bx_0$. We set $\lambda_{\textnormal{ini}} =10$ and use $q=5$ on MNIST and CIFAR-10 and
 $q=4$ on ImageNet
 for distortion fine-tuning in the post-attack phase. Comparing to Table \ref{table:ImageNet50}, the number of model queries can be further reduced since untargeted attacks only require the adversarial images to be classified as any class other than $t_0$ rather than classified as a specific class $t \neq t_0$.

\begin{table*}[t]
	\centering
	\caption{Performance evaluation of black-box untargeted attacks on different datasets. The per-pixel $L_{2}$ distortion thresholds are 0.004, 0.0015 and $5\times10^{-5}$ for MNIST, CIFAR-10 and ImageNet, respectively. }
	\label{table:untargeted}
	\begin{adjustbox}{max width=0.97\textwidth}	
		\begin{tabular}{ccccccccc}
			
			\hline
			Dataset
			& Method
			& \thead{Attack success\\rate (ASR)}
			& \thead{Mean query count\\(initial success)}
			& \thead{Mean query\\count reduction\\ratio (initial success)}
			& \thead{Mean per-pixel\\$L_{2}$ distortion\\(initial success)}
			& \thead{True positive\\rate (TPR)}
			& \thead{Mean query count\\with per-pixel $L_{2}$\\distortion $\leq$ threshold}\\
			\hline
			\multirow{2}{*}{MNIST}     & ZOO 	          & 100.00\%  &   7856.64 &  0.00\% & 3.79$\times10^{-3}$ & 100.00\% &  12392.96 \\
			& AutoZOOM-BiLIN   & 100.00\%  &     98.82 & 98.74\% & 4.21$\times10^{-3}$ & 100.00\% &    692.94 \\
			\hline
			\multirow{2}{*}{CIFAR-10}  & ZOO 	          & 100.00\%  &   3957.76 &  0.00\% & 5.78$\times10^{-4}$ & 100.00\% &   4644.60 \\
			& AutoZOOM-BiLIN   & 100.00\%  &      85.6 & 97.83\% & 6.47$\times10^{-4}$ & 100.00\% &    104.48 \\
			\hline
			\multirow{2}{*}{ImageNet}  & ZOO 	          &  94.00\%  & 271627.91 &  0.00\% & 2.32$\times10^{-5}$ &  91.49\% & 334901.58 \\
			& AutoZOOM-BiLIN   & 100.00\%  &   1695.27 & 99.37\% & 3.02$\times10^{-5}$ &  94.00\% &   4639.11 \\
			\hline
		\end{tabular}
	\end{adjustbox}
	\vspace{-4mm}
\end{table*}

\end{document}